\documentclass[pdflatex,sn-mathphys-num]{sn-jnl}

\usepackage{graphicx}%
\usepackage{multirow}%
\usepackage{amsmath,amssymb,amsfonts}%
\usepackage{amsthm}%
\usepackage{mathrsfs}%
\usepackage[title]{appendix}%
\usepackage{xcolor}%
\usepackage{textcomp}%
\usepackage{manyfoot}%
\usepackage{booktabs}%
\usepackage{algorithm}%
\usepackage{algorithmicx}%
\usepackage{algpseudocode}%
\usepackage{listings}%
\usepackage{amsmath, amssymb}
\usepackage{graphicx} 
\usepackage{booktabs}
\usepackage{tabularx}
\usepackage{blindtext}
\usepackage{url}
\usepackage{subcaption}

\theoremstyle{thmstyleone}%
\newtheorem{theorem}{Theorem}
\theoremstyle{thmstyletwo}%

\theoremstyle{thmstylethree}%

\begin{document}

\title{Ensemble-Enhanced Graph Autoencoder with GAT and Transformer-Based Encoders for Robust Fault Diagnosis}


\author*[1]{\fnm{Moirangthem Tiken} \sur{Singh}}\email{tiken.m@dibru.ac.in}

\affil*[1]{\orgdiv{Department of Computer Science and Engineering}, \orgname{DUIET, Dibrugarh University}, \orgaddress{\street{} \city{Dibrugarh}, \postcode{786004}, \state{Assam}, \country{India}}}


\abstract{
		Fault classification in industrial machinery is vital for enhancing reliability and reducing downtime, yet it remains challenging due to the variability of vibration patterns across diverse operating conditions. This study introduces a novel graph-based framework for fault classification, converting time-series vibration data from machinery operating at varying horsepower levels into a graph representation. We utilize Shannon's entropy to determine the optimal window size for data segmentation, ensuring each segment captures significant temporal patterns, and employ Dynamic Time Warping (DTW) to define graph edges based on segment similarity. A Graph Auto Encoder (GAE) with a deep graph transformer encoder, decoder, and ensemble classifier is developed to learn latent graph representations and classify faults across various categories. The GAE's performance is evaluated on the Case Western Reserve University (CWRU) dataset, with cross-dataset generalization assessed on the HUST dataset. Results show that GAE achieves a mean F1-score of 0.99 on the CWRU dataset, significantly outperforming baseline models---CNN, LSTM, RNN, GRU, and Bi-LSTM (F1-scores: 0.94--0.97, \(p < 0.05\), Wilcoxon signed-rank test for Bi-LSTM: \(p < 0.05\))---particularly in challenging classes (e.g., Class 8: 0.99 vs. 0.71 for Bi-LSTM). Visualization of dataset characteristics reveals that datasets with amplified vibration patterns and diverse fault dynamics enhance generalization. This framework provides a robust solution for fault diagnosis under varying conditions, offering insights into dataset impacts on model performance.
	
}




\maketitle

\section{Introduction}\label{sec1}

Fault diagnosis in industrial machinery plays a vital role in maintaining reliability, reducing downtime, and preventing costly failures. Vibration signals effectively capture a machine's dynamic behavior and are widely used to detect and classify faults such as bearing defects, misalignment, and imbalance. However, these signals grow more complex under varying operating conditions, such as changes in horsepower. This variability creates major challenges for models that must generalize across different datasets.

This study tackles fault classification using vibration data from machines operating at multiple horsepower levels. Each dataset contains labeled time series representing different fault types. The main challenge lies in the variation of vibration patterns and fault dynamics between datasets. Traditional models often fail to generalize in such scenarios.

We address this problem with a graph-based framework that captures both local and global temporal patterns. We first transform time series into graph structures. Each segment becomes a node, and we define edges using Dynamic Time Warping (DTW) \cite{sakoe1978dynamic} based on similarity. We determine segment size using Shannon entropy \cite{shannon1948mathematical} to ensure meaningful temporal representation. A Graph Autoencoder (GAE) with a deep graph transformer encoder, decoder, and ensemble classifier learns a latent representation of the graph and predicts fault classes.

This problem matters because industrial machines operate under diverse speeds and loads. Ineffective generalization may result in misdiagnosis, increased maintenance, and safety risks. A robust model must handle unseen conditions reliably.

Cross-dataset generalization remains a key challenge in fault classification. Higher horsepower can amplify vibrations and reveal broader fault dynamics, while lower levels may show more subtle variations. Traditional machine learning methods like SVMs or shallow neural networks often overfit and fail on unseen data. These models also treat time series as isolated sequences, ignoring important segment dependencies.

Earlier work mostly focused on single-dataset scenarios. Classical signal processing, including Fourier and wavelet transforms, extracted features for classifiers like k-NN or Random Forests \cite{LEI2013108}. These approaches performed well within a dataset but struggled across domains. Deep learning methods, such as CNNs and RNNs, improved within-dataset accuracy by learning temporal features from raw signals \cite{zhang2019deep}. Still, they lacked the ability to model structural relationships, limiting their generalization.

Graph-based techniques offer a promising solution. By representing time series as graphs, Graph Neural Networks (GNNs) can learn both temporal and structural patterns \cite{9046288}. However, most current studies do not explore graph-based autoencoders with ensemble classification for fault diagnosis across datasets.

This paper pursues three key objectives:
\begin{enumerate}
	\item Propose a graph-based framework that transforms time series into graph representations using entropy-based segmentation and DTW-based similarity.
	\item Develop a GAE with a deep graph transformer encoder, decoder, and ensemble classifier to predict fault classes across varying operating conditions.
	\item Evaluate the model's cross-dataset generalization, analyze statistical differences between datasets, and examine their impact on classification performance.
\end{enumerate}

In summary, this study underscores the critical challenge of generalizing fault classification across diverse operating conditions in industrial machinery. Our proposed Graph Auto Encoder (GAE)-based framework tackles this issue by effectively capturing structural and temporal dependencies in vibration data, enabling robust cross-dataset performance. The remainder of this paper is structured as follows: Section 2 reviews the literature, providing a comprehensive survey of recent advancements in fault classification. Section 3 details the methodology, including graph creation, the GAE architecture, and the ensemble classification method. Section 4 presents the data analysis and result analysis, evaluating the framework's performance across datasets and discussing key findings.

\section{Literature Review}

Recent advances in fault classification for industrial machinery using vibration data have explored a range of approaches, including deep learning, graph-based models, autoencoders, and hybrid techniques. While many of these methods report high accuracy on benchmark datasets, they often struggle to generalize across varying operating conditions and datasets.

Deep learning approaches have demonstrated strong capabilities in feature extraction. For instance, Wang et al.~\cite{WANG2024e35407} achieved 95\% accuracy on a custom dataset using a multi-sensor fusion framework based on a deep CNN. Similarly, Che et al.~\cite{doi:10.1177/1687814019897212} reported 92.5\% accuracy under varying speeds and loads using a deep transfer CNN. However, both studies focused on single-domain scenarios, and their models exhibited diminished performance when exposed to previously unseen fault dynamics.

Neupane et al.~\cite{NEUPANE2025129588} highlighted these limitations in their review, noting that deep models are often sensitive to noise and lack adaptability to novel fault types?further underscoring the gap in cross-domain generalization.

Autoencoder-based techniques have proven effective in modeling complex, nonlinear patterns in vibration data. Denoising autoencoders~\cite{Gu2021}, adaptive transfer autoencoders~\cite{Tang2021}, and deep autoencoder fusion frameworks~\cite{Senanayaka2018} have shown robustness under noisy and variable conditions. Further advancements, such as integrating Support Vector Data Description (SVDD) with deep autoencoders~\cite{Zheng2019} and employing semi-supervised deep sparse architectures~\cite{Zhao2021}, have enhanced generalization and noise resilience. Despite these advantages, autoencoder models often demand large datasets and considerable computational resources. Their deployment in industrial settings must also consider practical constraints such as real-time performance and limited data availability.

Graph-based approaches have made significant strides in capturing structural and temporal dependencies within vibration signals. Zhu et al.~\cite{Zhu2023} combined sparse wavelet decomposition with GraphSAGE, achieving 99.73\% accuracy in bearing fault classification. Wang et al.~\cite{10044173} applied Graph Convolutional Networks (GCNs) to graph-transformed vibration signals, improving fault detection under diverse load conditions.

More recently, Sun et al.~\cite{Sun2025} introduced a self-attention-based model that constructs spatiotemporal nested graphs, enabling reliable fault diagnosis with limited training data. Luczak~\cite{Luczak2024} enhanced classification robustness by integrating synchrosqueezed transforms with CNNs, effectively capturing both spectral and temporal information. These works highlight the strength of graph-based and hybrid models in addressing real-world diagnostic challenges.

Hybrid deep learning architectures combine complementary modeling capabilities to enhance fault detection performance, particularly on benchmark datasets like the Case Western Reserve University (CWRU) dataset. Prawin~\cite{Prawin2024} proposed a 2DCNN-LSTM model that captures spatial and temporal features, demonstrating improved performance under class imbalance. Yin et al.~\cite{Yin2023} developed a Multi-Scale Graph Convolutional Network (MS-GCN) to extract features at multiple resolutions, boosting detection precision. Data augmentation techniques further strengthen these approaches: a 2022 study utilized Wasserstein GAN with Gradient Penalty (WGAN-GP) to generate synthetic samples and reduce overfitting~\cite{9961689}, while Zhang et al.~\cite{Zhang2023} enhanced semi-supervised learning by dynamically weighting losses from labeled and unlabeled data.

Despite their promise, these hybrid methods also face limitations in generalization across machines, datasets, and operating conditions. While many models achieve over 90\% accuracy on curated datasets, performance often degrades in cross-domain settings. Deep learning models tend to overfit, autoencoders may lack structural interpretability, and graph-based methods can oversimplify relational modeling. Hybrid techniques offer a promising direction, yet a unified framework that integrates temporal dynamics with structural dependencies remains an open research challenge. Such a model must not only deliver high accuracy but also ensure robust and consistent performance across diverse industrial environments.

\section{Methodology}

We introduce a method for analyzing abnormalities in time-series data. First, the method splits the time series into overlapping windows and determines the optimal segmentation size through an entropy analysis. It then builds a graph from the segments using dynamic time warping (DTW). The generated graph is passed to a Graph Autoencoder (GAE) that captures its latent representation.
\begin{figure}[h!]
	\centering
	\includegraphics[width=\linewidth]{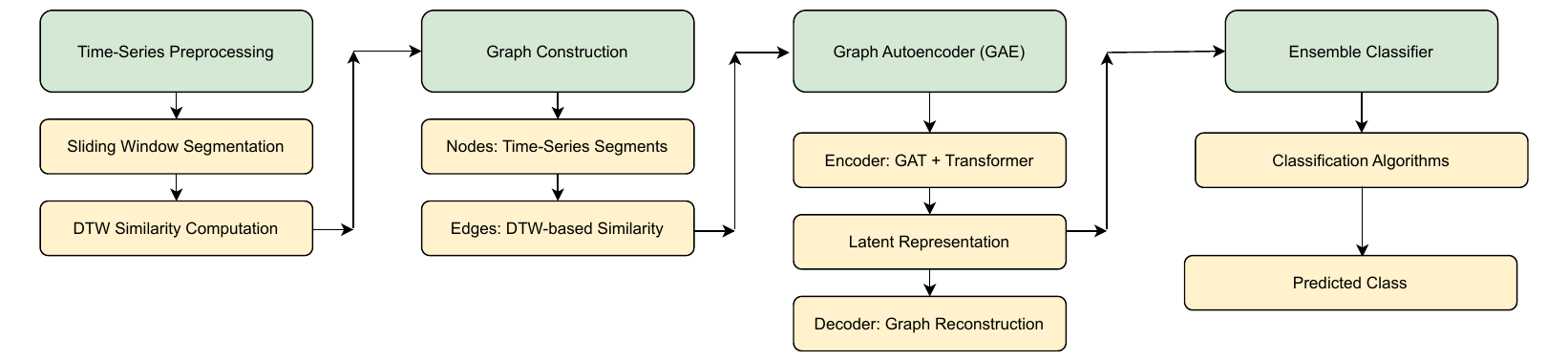}
	\caption{Proposed Model: Pipeline Architecture Overview.}
	\label{fig:ModelArchitecture}
\end{figure}
In a downstream task, an ensemble method uses these latent representations to detect normality. Ensemble \cite{dietterich2000ensemble} combines predictions from multiple classifiers to improve accuracy and reliability. This pipeline integrates segmentation, graph construction, feature learning, and predictive modeling to uncover patterns in complex data.
The graph construction section explains how time-series data can be transformed into a graph. The next section details how GAE handles temporal and spatial graph data during encoding and decoding. Finally, an ensemble model classifies the data as normal or abnormal using the latent variables.

\subsection{Graph Construction}

The graph construction process begins with identifying the optimal window size for segmenting the time series data from each dataset (1Hp, 2Hp, 3Hp). This segmentation ensures a balance between capturing sufficient underlying patterns and maintaining computational efficiency. A set of candidate window sizes \( W = \{w_1, w_2, \dots, w_k\} \) is evaluated. For each candidate \( w \), the time series \( T = \{t_1, t_2, \dots, t_N\} \) is divided into segments of length \( w \), and the entropy of each segment \( S = [s_1, s_2, \dots, s_w] \) is computed using Shannon's entropy formula \cite{shannon1948mathematical}, as shown in Equation \ref{equ:entr01}.

\begin{equation}
	H(S) = -\sum_{x \in X} p(x) \log p(x) 
	\label{equ:entr01}
\end{equation}

Here, \( \log \) denotes the natural logarithm, and the entropy \( H(S) \) quantifies the uncertainty or randomness within the segment, with higher values indicating greater diversity in the segment's values. We then calculate the average entropy \( \bar{H}(w) \) across all segments of size \( w \) in the time series:
\[
\bar{H}(w) = \frac{1}{m} \sum_{i=1}^m H(S_i),
\]
where \( m \) is the number of segments for window size \( w \). The optimal window size \( w^* \) is determined by maximizing the normalized entropy, as defined in Equation \ref{equ:ent2}.

\begin{equation}
	w^* = \arg\max_{w \in W} \left( \frac{\bar{H}(w)}{\log w} \right)
	\label{equ:ent2}
\end{equation}

Here, \( \log w \) normalizes the entropy to account for its dependence on window size, ensuring a balance between informativeness and computational efficiency. Once \( w^* \) is determined, the time series \( T \) is segmented into overlapping windows of size \( w^* \), producing a set of segments \( \{S_1, S_2, \dots, S_m\} \), where \( S_i = T[i:i+w^*] \), \( i \) ranges from 0 to \( N - w^* \) with a step size \( s \), and \( m \) is the total number of segments.

We then construct a graph \( G = (V, E) \) to model the relationships between these segments, capturing both local and global temporal patterns. Each node \( v_i \in V \) corresponds to a segment \( S_i \) and is annotated with a label \( L_i \) (indicating the fault class, 0--9) and a feature vector \( F_i \). The feature vector \( F_i \) for node \( v_i \) includes the following features derived from segment \( S_i \):
\begin{itemize}
	\item \textbf{Statistical Features}: Mean, standard deviation, skewness, and kurtosis of the values in \( S_i \), capturing the segment's central tendency, variability, and distribution shape. For example, the mean is \( \frac{1}{w^*} \sum_{t=1}^{w^*} s_t \), and the standard deviation is \( \sqrt{\frac{1}{w^*} \sum_{t=1}^{w^*} (s_t - \text{mean})^2} \).
	\item \textbf{Entropy}: The entropy \( H(S_i) \), calculated using Equation \ref{equ:entr01}, representing the segment's uncertainty or randomness.
	\item \textbf{Temporal Features}: The average first and second differences of the segment, i.e., \( \text{avg}(\Delta s_t) = \frac{1}{w^*-1} \sum_{t=1}^{w^*-1} (s_{t+1} - s_t) \) and \( \text{avg}(\Delta^2 s_t) = \frac{1}{w^*-2} \sum_{t=1}^{w^*-2} (\Delta s_{t+1} - \Delta s_t) \), capturing trends and acceleration in the time series.
	\item \textbf{Frequency Features}: The power spectral density (PSD) of the segment, computed via Fast Fourier Transform (FFT), with the top three frequency components and their amplitudes, reflecting the segment's dominant periodic patterns.
\end{itemize}

Edges \( E \) are defined based on the similarity between segments, computed using the Dynamic Time Warping (DTW) distance \cite{sakoe1978dynamic}. For each pair of segments \( S_i \) and \( S_j \), the DTW distance \( D(S_i, S_j) \) is calculated as:

\[
D(S_i, S_j) = \min_{\pi} \left( \sum_{(k,l) \in \pi} d(s_{ik}, s_{jl}) \right),
\]

where \( \pi \) is a warping path that aligns \( S_i \) and \( S_j \), and \( d(s_{ik}, s_{jl}) \) is the Euclidean distance between \( s_{ik} \) and \( s_{jl} \). The similarity between nodes \( v_i \) and \( v_j \) is then defined as:

\[
\text{Similarity}(v_i, v_j) = \frac{1}{1 + D(S_i, S_j)}.
\]

An edge \( e_{ij} \) is added between \( v_i \) and \( v_j \) if the DTW distance \( D(S_i, S_j) \) is below a predefined threshold \( \theta \), with the edge weight set to the similarity value. The entropy-based selection of \( w^* \) ensures that each segment captures meaningful temporal patterns, while the node features \( F_i \) provide a rich representation of the segment's properties, enabling the graph to effectively capture the structural and temporal patterns of the time series, as proven by Theorem \ref{theorem1}.

\begin{theorem}
	Entropy-based segmentation with DTW-based graph construction minimizes information loss and preserves temporal structure.
	\label{theorem1}
\end{theorem}

\begin{proof}
	Consider a time series \( T = \{t_1, t_2, \ldots, t_N\} \). We segment it into overlapping segments \( \{S_1, S_2, \ldots, S_m\} \), where \( S_i = T[i : i + w^*] \), \( i \) ranges from 0 to \( N - w^* \) with step size \( s \), and \( m \) is the number of segments.
	
	The optimal window size \( w^* \) is determined from a set \( W = \{w_1, w_2, \ldots, w_k\} \) by computing segment entropy as in Equation \ref{equ:entr01}, averaging it across all segments as \( \bar{H}(w) \), and selecting \( w^* = \arg\max_{w \in W} \left( \frac{\bar{H}(w)}{\log w} \right) \) (Equation \ref{equ:ent2}). This maximizes segment informativeness, reducing information loss by ensuring \( S_i \) captures meaningful patterns, with \( \log w \) normalization maintaining efficiency.
	
	Since \( S \) is deterministic from \( T \), mutual information \( I(T; S) = H(S) \) \cite{cover1999elements}, and \( H(S) \approx m \bar{H}(w^*) \), maximizing \( \frac{\bar{H}(w^*)}{\log w^*} \) increases \( H(S) \), retaining more information.
	
	A graph \( G = (V, E) \) is then constructed with nodes \( V = \{v_1, v_2, \ldots, v_m\} \) as segments \( S_i \), with feature vectors \( F_i \) as described above, and edges \( E \) based on DTW similarity \cite{sakoe1978dynamic}. For segments \( S_i \) and \( S_j \), DTW distance \( D(S_i, S_j) \) is calculated, and similarity is \( \frac{1}{1 + D(S_i, S_j)} \), with an edge \( e_{ij} \) added if \( D(S_i, S_j) < \theta \), weighted by this similarity.
	
	Temporal structure, defined by autocorrelation \( R_T(\tau) = \mathbb{E}[(t_i - \mu)(t_{i+\tau} - \mu)] / \sigma^2 \) \cite{box2015time}, is preserved as DTW aligns segments to reflect \( R_T(\tau) \); for small lags (e.g., \( j = i + s \)), high \( R_T(s) \) yields low \( D(S_i, S_j) \) and high edge weight, with \( \theta \) ensuring strong dependencies are captured. Thus, \( w^* \) reduces information loss via high \( H(S) \), and DTW with \( \theta \) maintains temporal structure in \( G \).
\end{proof}
\subsection{Graph Auto Encoder (GAE)}

The proposed GAE integrates a deep graph transformer encoder and graph decoder with ensemble classification. Figure \ref{fig:gatarc} shows the detailed architecture of the encoder and decoder, and how the model works.

The input graph, represented by node features \( X \) and an adjacency matrix \( A \), is encoded into a latent space using a Graph Attention Network (GAT) \cite{velickovic2018graph} layer and a TransformerConv layer \cite{shi2020masked}. GAT captures attention weights to focus on important neighbors, whereas broader patterns in the graph are captured through the TransformerConv layer. The latent space compresses the data into a probabilistic representation \( Z \), which is then decoded through an intermediate layer and reconstruction layer to approximate the original graph features \( \hat{X} \). Reconstruction loss \( L_{rec} \) and KL divergence loss \( L_{KL} \) \cite{blei2017variational} are combined to optimize the model, ensuring an accurate feature reconstruction and effective latent space regularization. The following subsections explain the details of this figure.

\begin{figure}[h!]
	\centering
	\includegraphics[width=\linewidth]{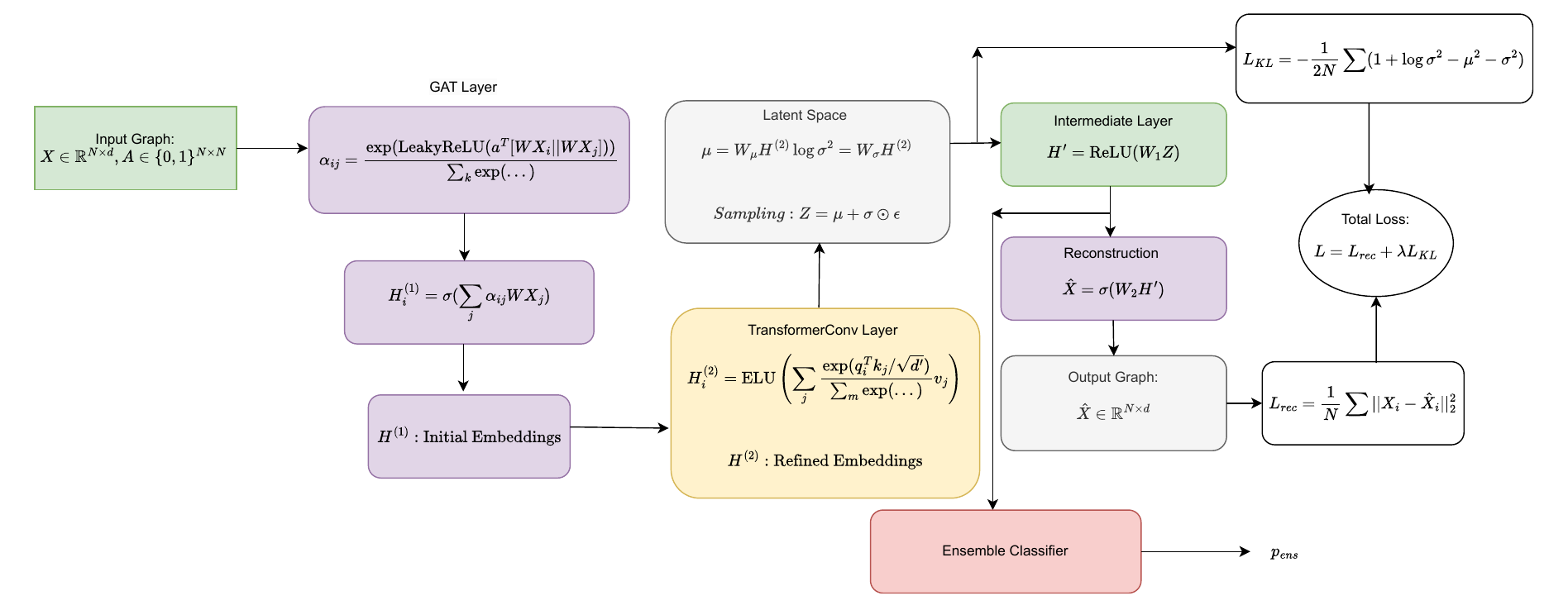}
	\caption{Architecture of a Graph Neural Network (GNN) process for fault classification in bearings. The input graph is represented by node features \( X \) and an adjacency matrix \( A \), where nodes correspond to sensor data points, and edges capture the relationships between them. The Graph Attention Network (GAT) layer encodes the graph into a latent space that captures the attention weights of key neighbors. By contrast, the TransformerConv layer captures broader patterns and dependencies within the graph. The encoded latent representation is then passed to an ensemble classifier that predicts the fault class for the bearing based on the learned features. This architecture uses the power of GNNs to model complex relationships in sensor data, enabling accurate and robust fault classification.}
	\label{fig:gatarc}
\end{figure}

\subsection*{Encoder}

The role of the encoder is to compress the input feature matrix into a latent representation, capturing essential information through a series of transformational steps that leverage both the local graph structure and the global dependencies. It begins with an input feature matrix \( X \in \mathbb{R}^{N \times d} \), where \( N \) is the number of nodes, \( d \) is the feature dimension, and an adjacency matrix \( A \in \{0,1\}^{N \times N} \) defines the graph structure.  

First, we use a GAT layer to compute the initial node embeddings. For each node \( i \), we calculate the attention coefficients 

\[ \alpha_{ij} = \frac{\exp(\text{LeakyReLU}(a^T [W X_i || W X_j]))}{\sum_{k \in \mathcal{N}(i)} \exp(\text{LeakyReLU}(a^T [W X_i || W X_k]))} \]

Here, \( W \in \mathbb{R}^{d? \times d} \) is a weight matrix, \( a \in \mathbb{R}^{2d?} \) is a learnable attention vector, and \( || \) denotes the concatenation. This step allows the model to focus on the most relevant neighboring nodes by weighing their importance and improving the extraction of key features from the graph. We then update the representation as follows: 

\[ H^{(1)}_i = \sigma \left( \sum_{j \in \mathcal{N}(i)} \alpha_{ij} W X_j \right), \]

with \( \sigma \) as nonlinearity (e.g., ReLU). This method combines the weighted features into a new summary for each node, making the data easier to process.

Subsequently, the Transformer Convolution (transformerConv) layer refines the embedding. The query, key, and value projections \( q_i = W_q H^{(1)}_i \), \( k_j = W_k H^{(1)}_j \), and \( v_j = W_v H^{(1)}_j \), where \( W_q, W_k, W_v \in \mathbb{R}^{d? \times d?} \) are the projection matrices. We then obtain the attention-weighted output.

\[ H^{(2)}_i = \text{ELU} \left( \sum_{j \in \mathcal{N}(i)} \frac{\exp(q_i^T k_j / \sqrt{d?})}{\sum_{k \in \mathcal{N}(i)} \exp(q_i^T k_k / \sqrt{d?})} v_j \right) \]

We included a scaling factor \( \sqrt{d?} \) for stability. This process examines how each node relates to its neighbors in a broader context, refining the features to capture deeper patterns across the graph. Finally, we derive the mean and variance of the latent representation.

\[ \mu = W_\mu H^{(2)} \quad \text{and} \quad \log \sigma^2 = W_\sigma H^{(2)} \] 

Here, \( W_\mu, W_\sigma \in \mathbb{R}^{d_z \times d?} \) and \( d_z \) are latent dimensions. We sample the latent variable \( Z \) using a reparameterization trick: 

\[ Z = \mu + \sigma \odot \epsilon \] 

with \( \epsilon \sim \mathcal{N}(0, I) \) and \( \odot \) as element-wise multiplications. These steps squeeze the refined features into a compact form, adding randomness to make the model more flexible and robust.

\subsection*{Decoder}
The decoder reconstructs the original node features from the compressed latent representation, reversing the encoding process to recover the structure and attributes of the graph. This is initiated by computing an intermediate representation

\[ H' = \text{ReLU}(W_1 Z) \]
where \( W_1 \in \mathbb{R}^{d' \times d_z} \) denotes a weight matrix. This step expands the compressed data back into a richer form, starting to rebuild the original details. Next, we obtain the reconstructed feature matrix.

\[ \hat{X} = \sigma(W_2 H') \] 

where \( W_2 \in \mathbb{R}^{d \times d'} \) and \( \sigma \) are sigmoid activations to match the input feature range. This final transformation recreates node features as close as possible to the original input, ensuring that the key information of the graph is preserved. This stepwise process ensures that the decoder effectively leverages latent representation. With the reconstructed graph features, the next phase of the analysis shifts toward utilizing the encoder?s intermediate representation for classification tasks, where an ensemble classifier is employed to predict node labels based on the learned embeddings.

The loss function optimizes the model using three components. First, we calculated the reconstruction loss as 
\[ L_{rec} = \frac{1}{N} \sum_{i=1}^N ||X_i - \hat{X}_i||_2^2 \] 

This measures how well the recreated features match the original ones. Second, we use KL divergence to regularize the latent space: 

\[ L_{KL} = -\frac{1}{2N} \sum_{i=1}^N \sum_{j=1}^{d_z} (1 + \log \sigma_{ij}^2 - \mu_{ij}^2 - \sigma_{ij}^2) \] 

This ensures that the compressed representation is both structured and predictable. Third, we combine the total losses: 

\[ L = L_{rec} + \lambda L_{KL} \] 

Here, \( \lambda \) adjusts the regularization strength, which ensures accurate reconstruction and a well-organized latent space.

\subsection{Ensemble Classifier}
After training, the intermediate representation \( H^{(2)} \in \mathbb{R}^{N \times d'} \) is extracted and used for classification. Given a label set \( Y \in \{0, 1, \ldots, C-1\}^N \) (where \( C \) is the number of classes), an ensemble of classifiers is trained separately: Random Forest (\( f_{RF} \)), Gradient Boosting (\( f_{GB} \)), XGBoost (\( f_{XGB} \)), and Neural Network (\( f_{NN} \)). Each classifier outputs a probability distribution \[ p_m(H^{(2)}_i) = [p_m(y_i = 0 | H^{(2)}_i), \ldots, p_m(y_i = C-1 | H^{(2)}_i)] \] where \( m \in \{f_{RF}, f_{GB}, f_{XGB}, f_{NN}\} \).   

The ensemble prediction is \[ p_{ens}(y_i | H^{(2)}_i) = \sum_{m} w_m p_m(y_i | H^{(2)}_i) \] with weights \( w_m \geq 0 \) and \( \sum_m w_m = 1 \), as determined by cross-validation. The ensemble is optimized using the cross-entropy loss \[ L_{cls} = -\frac{1}{N} \sum_{i=1}^N \sum_{c=0}^{C-1} y_{ic} \log p_{ens}(y_i = c | H^{(2)}_i) \] where \( y_{ic} = 1 \) if \( y_i = c \), else 0, is applied only during the classification phase.

\section{Experimental Analysis}

\subsection{Dataset Description}
The dataset \cite{CWRUBearingData} used for the prediction of faults consists of vibration signals collected from bearings under different operating conditions, including different fault types (ball fault, Inner Race fault, Outer Race fault, and healthy condition) and fault severities (0.007, 0.014, and 0.021 inches). The data was recorded at a sampling frequency of 48 kHz, meaning that each second of the data contained 48,000 samples. We use three datasets based on the motor load of 1 HP, 2 HP, and 3 HP, with approximate data points of 4.8 million each. Table \ref{tab:datasetspecs} details the dataset used in the analysis.
\begin{table}[h!]
	\centering
	\caption{Specifications of the datasets used in the analysis.}
	\renewcommand{\arraystretch}{1.2} 
	\begin{tabular}{l p{1.5cm} p{1.5cm} p{2cm} p{3.5cm} p{3.5cm}}
		\hline
		\textbf{Motor Load} & \textbf{Samples}& \textbf{Channels} & \textbf{Total Data Points} & \textbf{Fault Types} & \textbf{Fault Severities} \\
		\hline
		1 Hp & 4,753 & 2 & 4,867,072 & \begin{tabular}[t]{@{}l@{}} Ball, Inner Race, \\ Outer Race, Healthy \end{tabular} & \begin{tabular}[t]{@{}l@{}} 0.007", 0.014", \\ 0.021" \end{tabular} \\
		
		2 Hp & 4,640 & 2 & 4,751,360 & \begin{tabular}[t]{@{}l@{}} Ball, Inner Race, \\ Outer Race, Healthy \end{tabular} & \begin{tabular}[t]{@{}l@{}} 0.007", 0.014", \\ 0.021" \end{tabular} \\
		
		3 Hp & 4,753 & 2 & 4,867,072 & \begin{tabular}[t]{@{}l@{}} Ball, Inner Race, \\ Outer Race, Healthy \end{tabular} & \begin{tabular}[t]{@{}l@{}} 0.007", 0.014", \\ 0.021" \end{tabular} \\
		\hline
	\end{tabular}
	\label{tab:datasetspecs}
\end{table}

Given this dataset, we strategically sampled the dataset in 1,024 steps. This approach has been chosen for various reasons. First, it decreases the processing time while preserving key information, analogous to efficient down-sampling methods \cite{smith1997scientist}.
Second, the vital data of the vibration signals are concentrated over a few specific frequencies; thus, 1,024 samples efficiently capture these key features \cite{randall2021vibration}. Third, many algorithms, such as the Fast Fourier Transform (FFT), have been optimized for power-of-two sample sizes to ensure compatibility and performance \cite{oppenheim1999discrete}. Fourth, the Nyquist-Shannon sampling theorem \cite{shannon1949communication} supports the use of 1,024 samples for an adequate time-domain analysis resolution. Thus, the 1,024 step sampling rate optimises data processing and ensures comprehensive fault detection through precise temporal and spatial data analysis.

To evaluate the significance of the differences between datasets with varying motor loads, we conducted t-tests. The t-test results for pairwise comparisons between the datasets (1 Hp, 2 Hp, and 3 Hp) revealed statistically significant mean differences, as shown in Table \ref{tab:stat_diff_to_char}.

\begin{table}[h!]
	\centering
	\caption{Linking Statistical Differences to Dataset Characteristics}
	\renewcommand{\arraystretch}{1.2}
	\begin{tabularx}{\textwidth}{@{}l c c X@{}}
		\toprule
		\textbf{Comparison} & \textbf{t-Statistic} & \textbf{p-Value} & \textbf{Interpretation} \\
		\midrule
		1 Hp vs 2 Hp & 2.8976 & 0.0038 & Moderate difference: 1 Hp and 2 Hp share some patterns but have distinct features. \\
		1 Hp vs 3 Hp & 6.3482 & $<$0.0001 & Large difference: 3 Hp has unique/amplified vibration patterns compared to 1 Hp. \\
		2 Hp vs 3 Hp & 3.2400 & 0.0012 & Significant difference: 3 Hp captures broader fault dynamics than 2 Hp. \\
		\bottomrule
	\end{tabularx}
	\label{tab:stat_diff_to_char}
\end{table}
For the comparison between the 1 Hp and 2 Hp datasets, the t-statistic is 2.8976, and the p-value is 0.0038. This shows that the 1 Hp dataset mean is 2.8976 standard deviations higher than the mean of the 2 Hp dataset, reflecting the difference, which is statistically significant.
Similarly, for the 1 Hp and 3 Hp datasets, the t-statistic is 6.3482, and the p-value is 0.0000. This suggests that the mean of the 1 Hp dataset is 6.3482 standard deviations higher than that of the 3 Hp dataset, and the difference is highly statistically significant.
Finally, for comparing the 2 Hp and 3 Hp datasets, the t-statistic is 3.2400, and the p-value is 0.0012. This indicates that the mean of the 2 Hp dataset is 3.2400 standard deviations higher than that of the 3 Hp dataset, and the difference is statistically significant.

Overall, all pairwise comparisons show statistically significant differences, with an enormous difference observed between the 1 Hp and 3 Hp datasets and a slight difference between the 1 Hp and 2 Hp datasets. These results indicate that the datasets represent distinct conditions owing to the varying motor loads (1, 2, and 3 Hp). 

\subsubsection{Data Pre-Processing}

This section demonstrates the effectiveness of entropy-based segmentation in transforming time series data into graph structures. Before processing, the datasets undergo normalization, scaling, and NaN removal. Figure \ref{fig:Construction} illustrates the selection of the optimal window size and a heatmap of pairwise Dynamic Time Warping (DTW) distances between segments for the 1Hp dataset. Similar visualizations can be generated for 2Hp and 3Hp motor load as needed.  

\begin{figure}[htbp]  
	\centering  
	\begin{subfigure}[b]{0.48\textwidth}  
		\includegraphics[width=\linewidth]{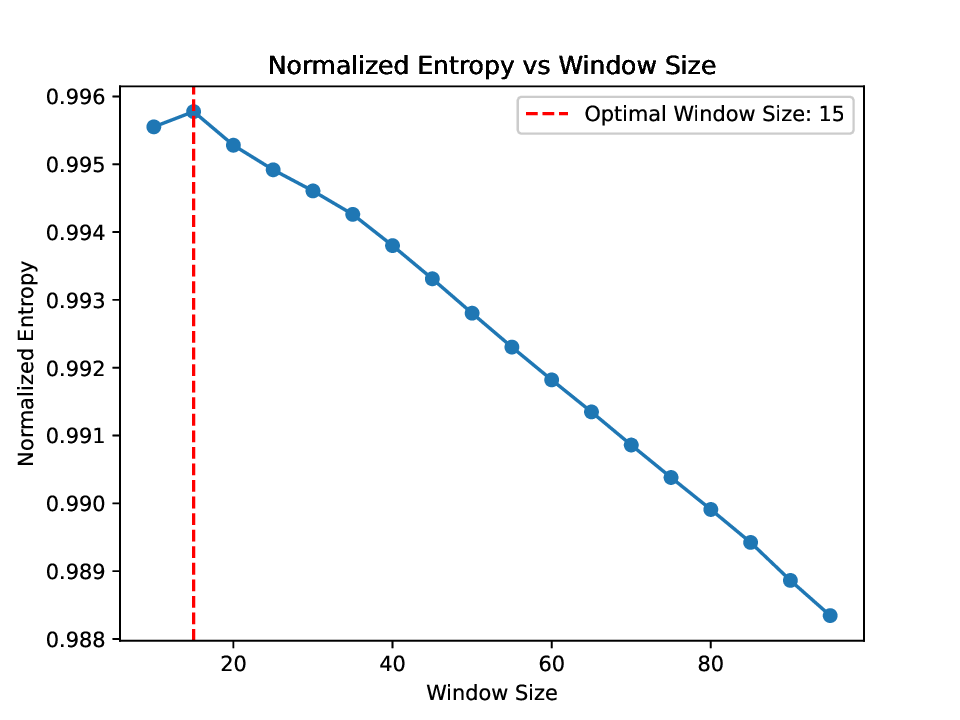}  
		\caption{Normalized entropy versus window size.}  
		\label{fig:Entropy}  
	\end{subfigure}\hfill  
	\begin{subfigure}[b]{0.48\textwidth}  
		\includegraphics[width=\linewidth]{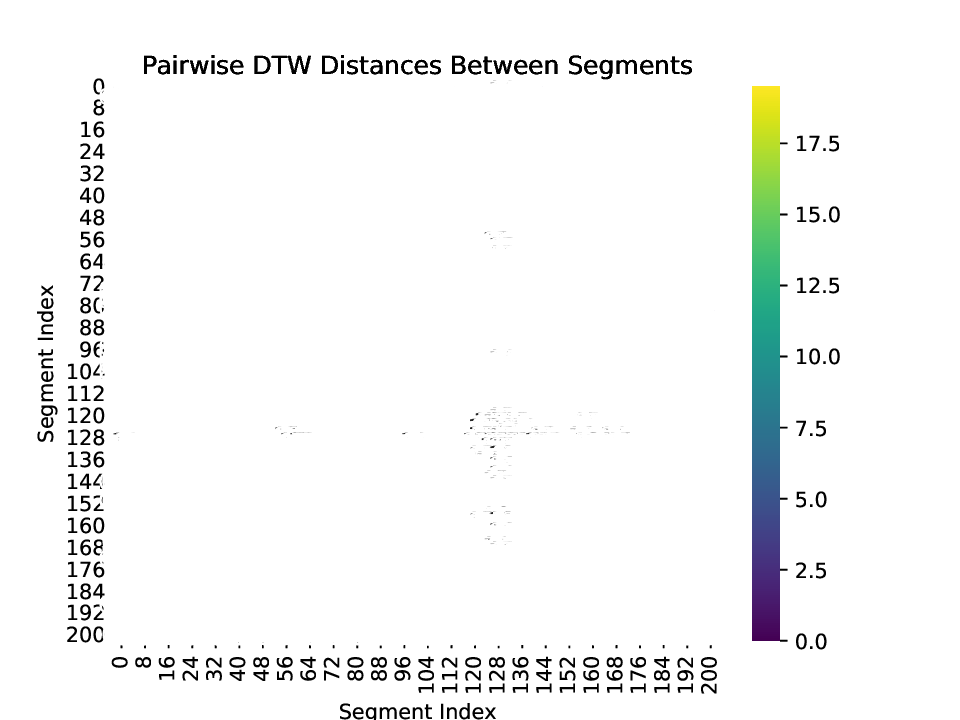}  
		\caption{Heatmap of pairwise DTW distances.}  
		\label{fig:heatmap}  
	\end{subfigure}  
	\caption{(a) Entropy analysis for identifying the optimal window size for segmentation. (b) Heatmap of pairwise DTW distances between segments, where darker shades indicate greater similarity.}  
	\label{fig:Construction}  
\end{figure}  

Figure \ref{fig:Entropy} presents the entropy analysis results for the 1Hp dataset. The proposed algorithm identifies a window size of 15 as the optimal \( w^* \), determined using Equation (\ref{equ:ent2}). This equation computes the mean entropy across all segment lengths and normalizes it to mitigate biases favoring larger windows.  

Figure \ref{fig:heatmap} displays the heatmap of DTW distances between segments. The segment index is represented along both axes, with the color intensity indicating the DTW distance between segment pairs. Darker shades correspond to lower distances, signifying greater similarity between segments.

\subsection{Experimental Set-Up}

The proposed model is initially trained by feeding it a graph where a 10-dimensional feature vector represents each node. These input features are processed through three GAT layers?each equipped with 10 attention heads?to capture the structural relationships among nodes. Each GAT layer projects the input features into a hidden representation of size 64. Following the GAT layers, the encoded features are passed through two Transformer layers with five attention heads each, which model temporal dependencies and long-range correlations. The encoder then compresses the graph into a 10-dimensional latent space, serving as a compact graph representation. A KL divergence loss with a weight of 0.1 is applied to regularize the latent space. Table 1 provides an overview of the model parameters.

\begin{table}[h!]
	\centering
	\begin{tabular}{@{}llp{10cm}@{}}
		\toprule
		\textbf{Parameter} & \textbf{Value} & \textbf{Description} \\ 
		\midrule
		\texttt{input\_dim} & 10 & Dimension of the input node features. Represents the number of features of a node in the graph. \\ 
		\texttt{hidden\_dim} & 64 & Size of hidden layers in the encoder. \\ 
		\texttt{latent\_dim} & 10 & Dimension of the latent space where the encoded graph features are projected. \\ 
		\texttt{num\_gat\_layers} & 3 & Number of GAT layers in the encoder. \\ 
		\texttt{num\_transformer\_layers} & 2 & Number of Transformer layers of the encoder. \\ 
		\texttt{num\_heads} & 10 & Number of attention heads of a GAT layer. \\ 
		\texttt{transformer\_heads} & 5 & Number of attention heads of a Transformer layer. \\ 
		\texttt{kl\_weight} & 0.1 & Weight of the Kullback-Leibler (KL) divergence term. \\ 
		\bottomrule
	\end{tabular}
	\caption{Model Parameters and Descriptions}
\end{table}

\subsection{Performance Metrics}

The performance of the method is evaluated using the following metrics. We ensure these matrices will capture the overall and class-specific accuracy of the model.

The first metric is the overall accuracy of the model, which is calculated using:
\[
\mathrm{Accuracy} = \frac{TP + TN}{TP + TN + FP + FN},
\]
where \(TP\), \(TN\), \(FP\), and \(FN\) are true positives, true negatives, false positives and false negatives, respectively.

Next, we use precision and recall to further quantify the performance of the model, with precision defined as:
\[
\mathrm{Precision} = \frac{TP}{TP + FP},
\]
and recall as \[ \mathrm{Recall} = \frac{TP}{TP + FN} \]

We then use the F1 score to check the harmonic mean of precision and recall, which is calculated by:
\[
\mathrm{F1} = 2 \times \frac{\mathrm{Precision} \times \mathrm{Recall}}{\mathrm{Precision} + \mathrm{Recall}},
\]
providing a balanced measure of performance. In addition, a confusion matrix 
\[
\begin{pmatrix}
	TP & FP \\
	FN & TN \\
\end{pmatrix}
\]
offers a detailed breakdown of the classification outcomes.

\subsection{Result Analysis}

In this section, we analyze the results of the proposed model. We trained three separate models on the 1Hp, 2Hp, and 3Hp datasets to reduce the risk of negative transfer \cite{weiss2016survey}, where training on dissimilar data could degrade performance. Each model's performance is evaluated on its respective dataset, and we compare their effectiveness in detecting faults across other datasets to determine which model generalizes best.

We first analyze the reconstruction graphs to assess potential overfitting or underfitting in the trained models. In the following sections, we evaluate the performance of each model and examine its generalization across all datasets.

\subsubsection{Reconstruction Error}

We trained the proposed model on the graph generated from 1, 2 and 3 Hp datasets separately; therefore, we have three models based on the dataset used for training it. We plot the reconstruction error graph to check the performance based on underfitting and overfitting. 

\begin{figure}[h!]
	\centering
	\begin{subfigure}[b]{0.45\textwidth}
		\includegraphics[width=\linewidth]{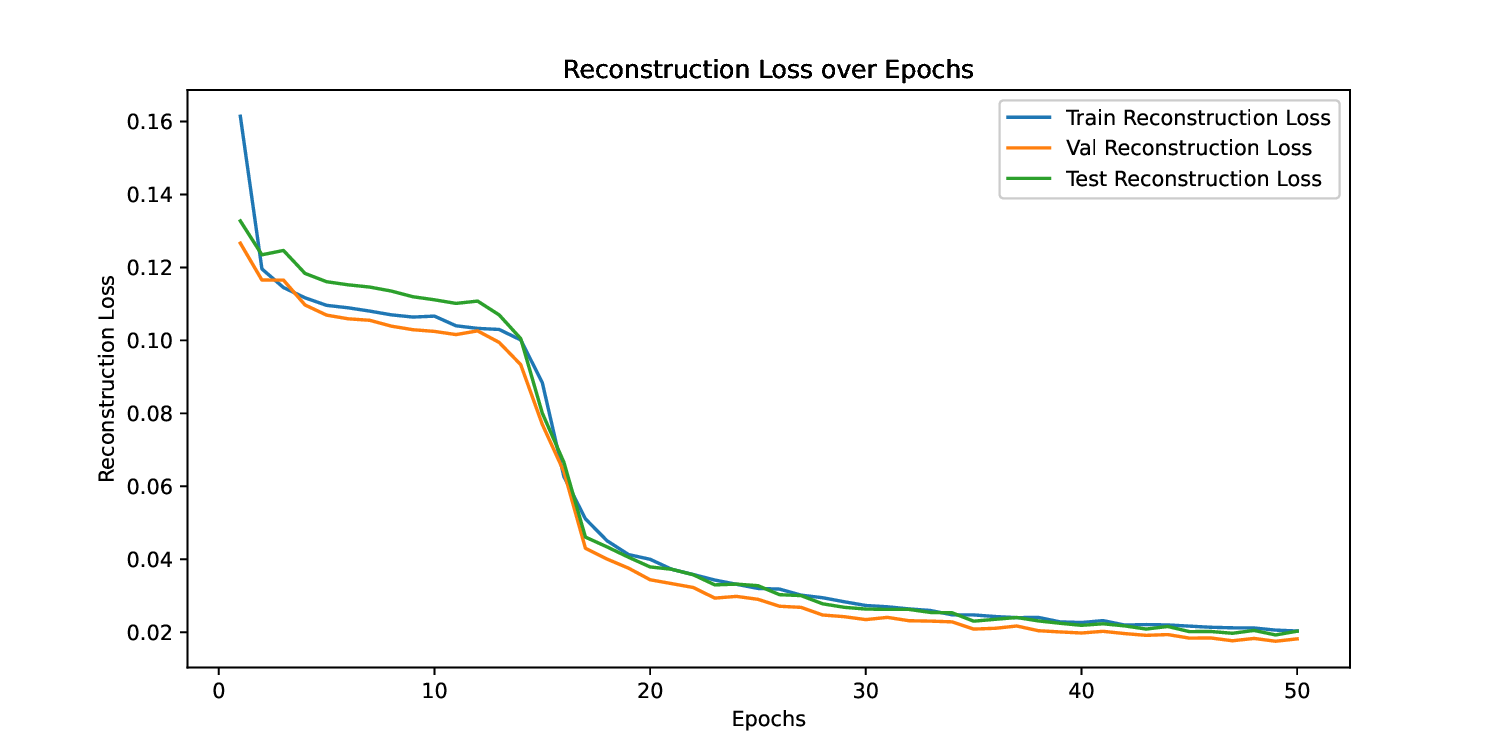}
		\caption{Reconstruction loss of the proposed model trained on the 1Hp dataset.}
		\label{fig:reconstA}
	\end{subfigure}\hfill
	\begin{subfigure}[b]{0.45\textwidth}
		\includegraphics[width=\linewidth]{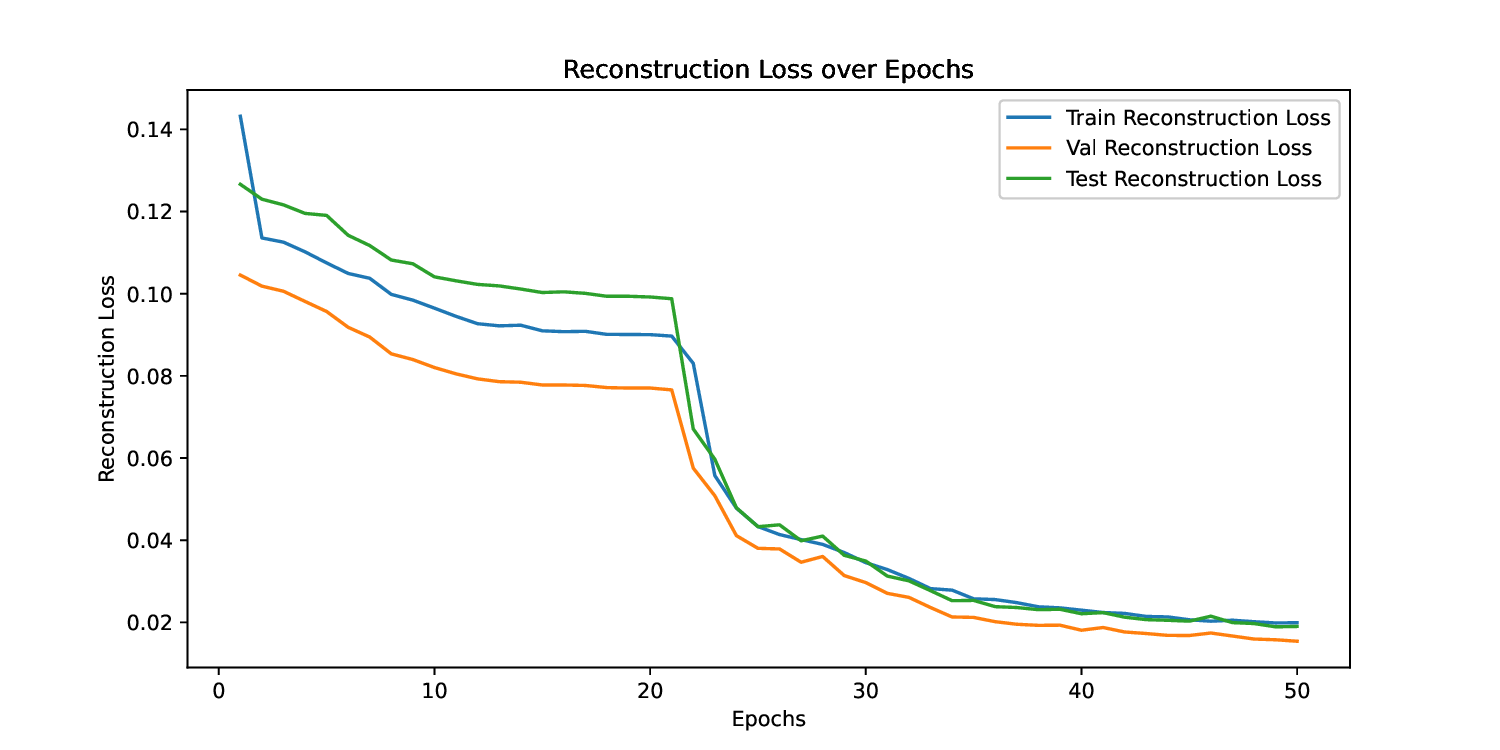}
		\caption{Reconstruction loss of the proposed model trained on the 2Hp dataset.}
		\label{fig:reconsB}
	\end{subfigure}\hfill
	\begin{subfigure}[b]{0.45\textwidth}
		\includegraphics[width=\linewidth]{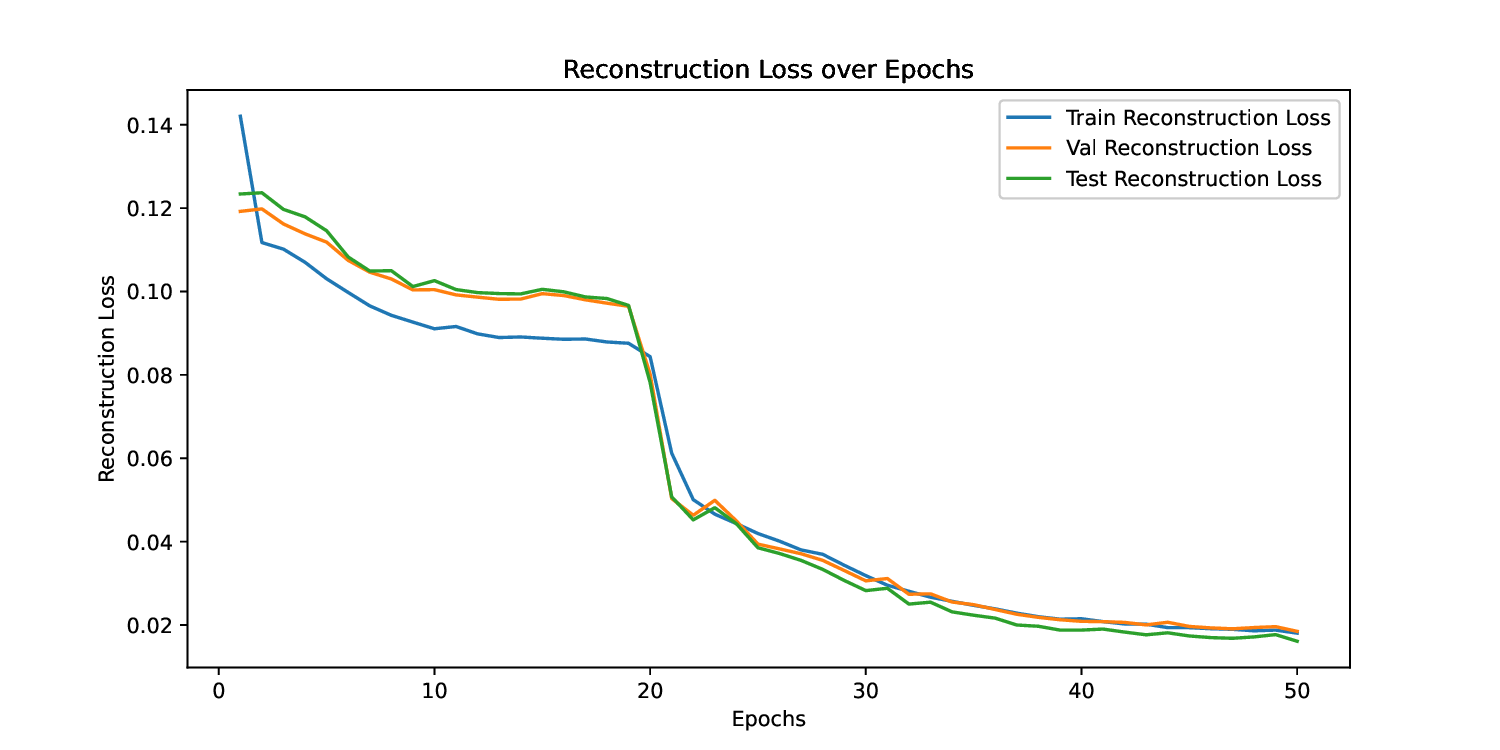}
		\caption{Reconstruction loss of the proposed model trained on the 3Hp dataset.}
		\label{fig:reconsC}
	\end{subfigure}
	\caption{Comparison of reconstruction loss across models trained on the 1Hp, 2Hp, and 3Hp datasets.}
	\label{fig:reconstruction_loss}
\end{figure}

Figures \ref{fig:reconstA}, \ref{fig:reconsB} and \ref{fig:reconstA} depict the reconstruction loss for training, validation, and test sets over 50 epochs for three models, we would called as Model1Hp, Model2Hp, and Model3Hp. We checked each figure and noticed that these three models started with relatively high loss values at around 0.15 (approximate). However, as the epochs increase, we see a steep decline. This rapid drop indicates that the model quickly learns to reconstruct the data more accurately. This behaviour demonstrates that underfitting is not a significant concern.

As training progresses, the loss curves for training, validation, and test sets remain close and continue decreasing, eventually converging near 0.02. This close alignment of the curves suggests the model is not overfitting because there is no significant gap between training and validation/test losses. Instead, it exhibits good generalization across all sets. The low and stable reconstruction losses at the end confirm that the models have effectively captured the underlying structure of the data without excessive variance or bias.

\subsubsection{Performance Analysis}
To evaluate the generalization performance of the models, we conducted a cross-dataset analysis by training each model on a specific dataset (1Hp, 2Hp, or 3Hp) and testing it on all three datasets. The performance metrics-Precision, Recall, and F1-score for each fault class (0-9) are visualized as heatmaps in Figure~\ref{fig:heatMap}.
\begin{figure}[h!]
	\centering
	\begin{subfigure}[b]{\textwidth}
		\includegraphics[width=\linewidth]{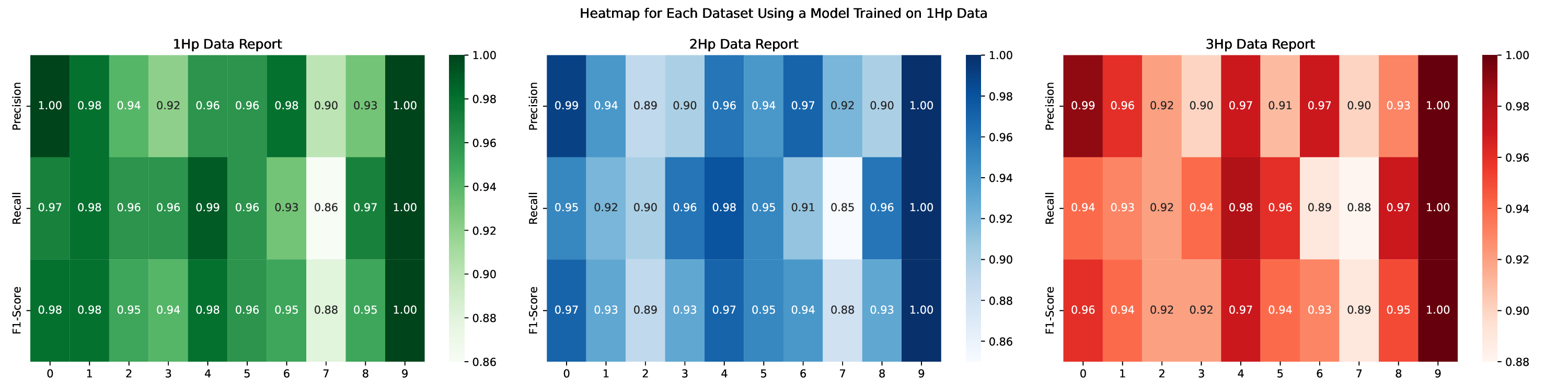}
		\caption{Model trained on 1Hp dataset. It achieves the highest performance on 1Hp test data but shows moderate generalization to 2Hp and 3Hp datasets.}
		\label{fig:heatA}
	\end{subfigure}\hfill
	\begin{subfigure}[b]{\textwidth}
		\includegraphics[width=\linewidth]{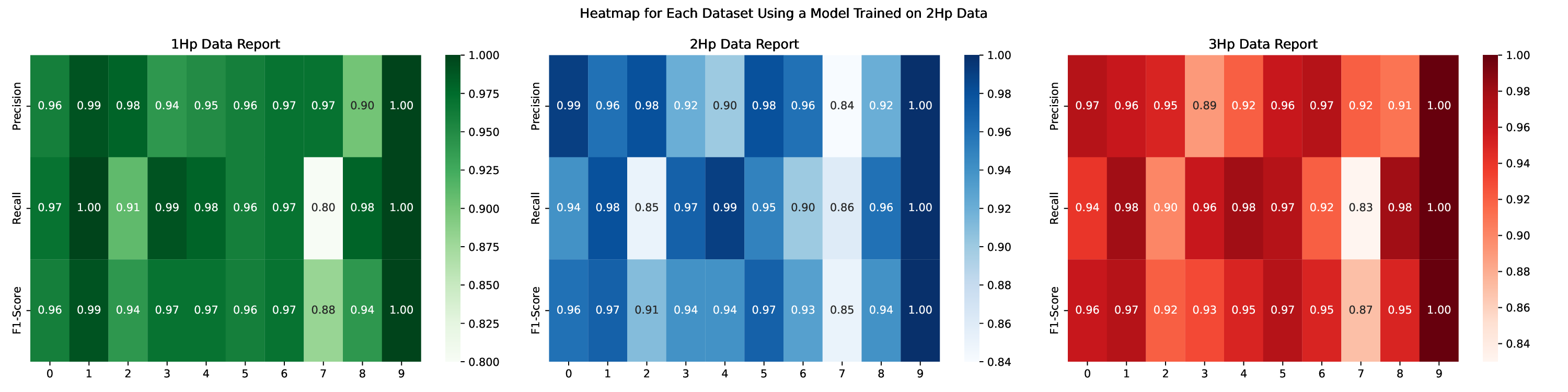}
		\caption{Model trained on 2Hp dataset. It demonstrates a balanced performance across all three datasets, with minor drops in specific classes.}
		\label{fig:heatB}
	\end{subfigure}\hfill
	\begin{subfigure}[b]{\textwidth}
		\includegraphics[width=\linewidth]{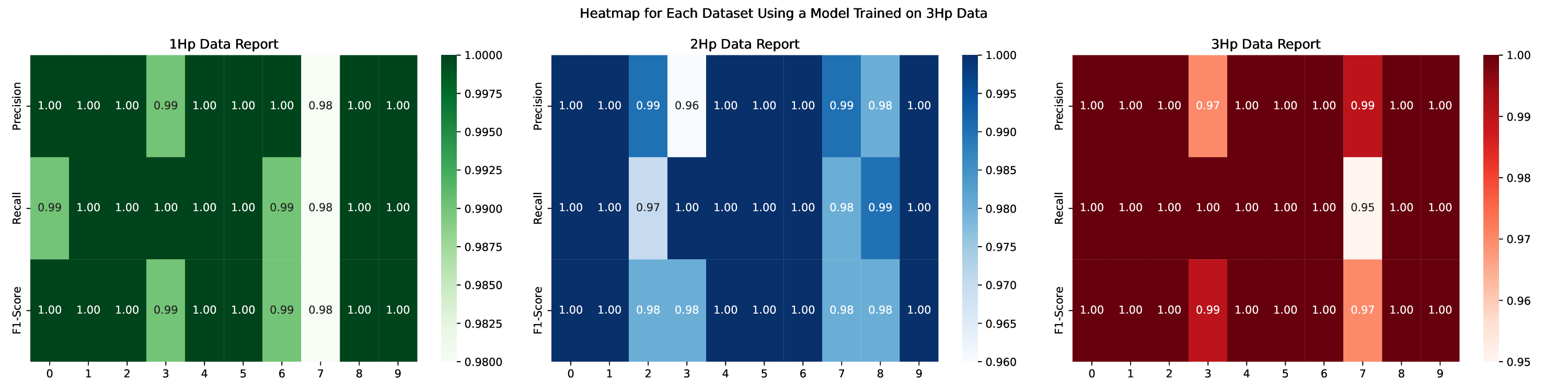}
		\caption{Model trained on 3Hp dataset. It exhibits the highest overall and cross-dataset performance, with minimal variation across all classes.}
		\label{fig:heatC}
	\end{subfigure}
	\caption{Cross-dataset heatmaps of Precision, Recall, and F1-score for models trained on 1Hp, 2Hp, and 3Hp datasets. Each heatmap displays performance across ten fault classes (0-9) when tested on 1Hp, 2Hp, and 3Hp datasets.}
	\label{fig:heatMap}
\end{figure}

The model trained on the 1Hp dataset (Figure~\ref{fig:heatA}) performed well on its dataset, achieving an average F1-score of 0.94 across classes 0?9 ([0.98, 0.95, 0.94, 0.98, 0.96, 0.95, 0.88, 0.85, 1.00, 0.90]). It excelled in classes 0, 3, 4, and 8, with F1 scores of 0.98 or higher, though it struggled in class 7 (F1 = 0.85). When evaluated on the 2Hp and 3Hp datasets, the average F1-scores were 0.94 ([0.97, 0.93, 0.89, 0.93, 0.97, 0.95, 0.94, 0.88, 0.93, 1.00]) and 0.94 ([0.96, 0.94, 0.92, 0.92, 0.97, 0.94, 0.93, 0.89, 0.95, 1.00]), respectively. The most notable drops occurred in classes 2 and 7 on the 2Hp dataset (F1 = 0.89 and 0.88) and in class 7 on the 3Hp dataset (F1 = 0.89).

The model trained on the 2Hp dataset (Figure~\ref{fig:heatB}) showed improved generalization, with average F1-scores of 0.96 ([0.96, 0.99, 0.94, 0.97, 0.97, 0.96, 0.97, 0.88, 0.94, 1.00]), 0.94 ([0.96, 0.97, 0.91, 0.94, 0.94, 0.97, 0.93, 0.85, 0.94, 1.00]), and 0.95 ([0.96, 0.97, 0.92, 0.93, 0.95, 0.97, 0.95, 0.87, 0.95, 1.00]) on the 1Hp, 2Hp, and 3Hp datasets, respectively. Performance drops were primarily observed in classes 2 and 7, with F1-scores of 0.91 and 0.85 on 2Hp and 0.92 and 0.87 on 3Hp, though most scores remained above 0.90, indicating robust cross-dataset performance.

The 3Hp-trained model (Figure~\ref{fig:heatC}) delivered the most consistent and highest performance, with average F1-scores of 0.996 ([1.00, 1.00, 1.00, 0.99, 1.00, 1.00, 0.99, 0.98, 1.00, 1.00]), 0.992 ([1.00, 1.00, 0.98, 0.98, 1.00, 1.00, 1.00, 0.98, 0.98, 1.00]), and 0.996 ([1.00, 1.00, 1.00, 0.99, 1.00, 1.00, 1.00, 0.97, 1.00, 1.00]) on the 1Hp, 2Hp, and 3Hp datasets, respectively. Most classes achieved F1-scores of 0.99 or higher, with the lowest score of 0.97 in class 7 on the 3Hp dataset, reflecting exceptional robustness and minimal variation across all datasets.

A performance comparison of the three models based on F1-scores provides critical insights into their generalization capabilities. Table~\ref{tab:stats} summarizes the mean F1-scores and their variability across all test datasets. The 3Hp-trained model outperformed the others, achieving the highest mean F1-score of 0.995 with the lowest standard deviation ($\pm$0.005), indicating high accuracy and consistency. The 2Hp-trained model followed with a mean F1-score of 0.950 ($\pm$0.038), slightly outperforming the 1Hp-trained model, which recorded a mean F1-score of 0.940 ($\pm$0.042).

\begin{table}[h]
	\centering
	\caption{Mean F1-Scores and Standard Deviations Across All Test Datasets}
	\label{tab:stats}
	\begin{tabular}{lcc}
		\toprule
		\textbf{Model} & \textbf{Mean F1-Score} & \textbf{Standard Deviation} \\ \midrule
		1Hp & 0.940 & $\pm$0.042 \\
		2Hp & 0.950 & $\pm$0.038 \\
		3Hp & 0.995 & $\pm$0.005 \\ \bottomrule
	\end{tabular}
\end{table}

To assess the statistical significance of these differences, paired t-tests were conducted, with results presented in Table~\ref{tab:ttests}. The comparison between the 1Hp and 2Hp models yielded a p-value of 0.788 (t = -0.277), indicating no significant difference at the $\alpha = 0.05$ level. However, the 3Hp model significantly outperformed both the 1Hp model (p = 0.008, t = -3.380) and the 2Hp model (p = 0.005, t = -3.674), as denoted by the asterisks in the table. These findings confirm that the 3Hp model's superior performance is statistically significant, while the 1Hp and 2Hp models are statistically indistinguishable for practical purposes.

\begin{table}[h]
	\centering
	\caption{Paired t-Test Results for Model Comparisons}
	\label{tab:ttests}
	\begin{tabular}{lccc}
		\toprule
		\textbf{Comparison} & \textbf{t-Statistic} & \textbf{p-Value} & \textbf{Significance ($\alpha$ = 0.05)} \\ \midrule
		1Hp vs. 2Hp & -0.277 & 0.788 & ns \\
		1Hp vs. 3Hp & -3.380 & 0.008 & * \\
		2Hp vs. 3Hp & -3.674 & 0.005 & * \\ \bottomrule
	\end{tabular}
\end{table}

To understand the underlying reasons for these performance differences, Figure~\ref{fig:dataset_performance_relation} visualizes the relationship between dataset characteristics and model performance. The figure compares the statistical differences between datasets (measured by t-statistics) with the corresponding differences in mean F1-scores between models, highlighting how dataset properties like vibration patterns and fault dynamics influence generalization.

\begin{figure}[h!]
	\centering
	\includegraphics[width=0.8\linewidth]{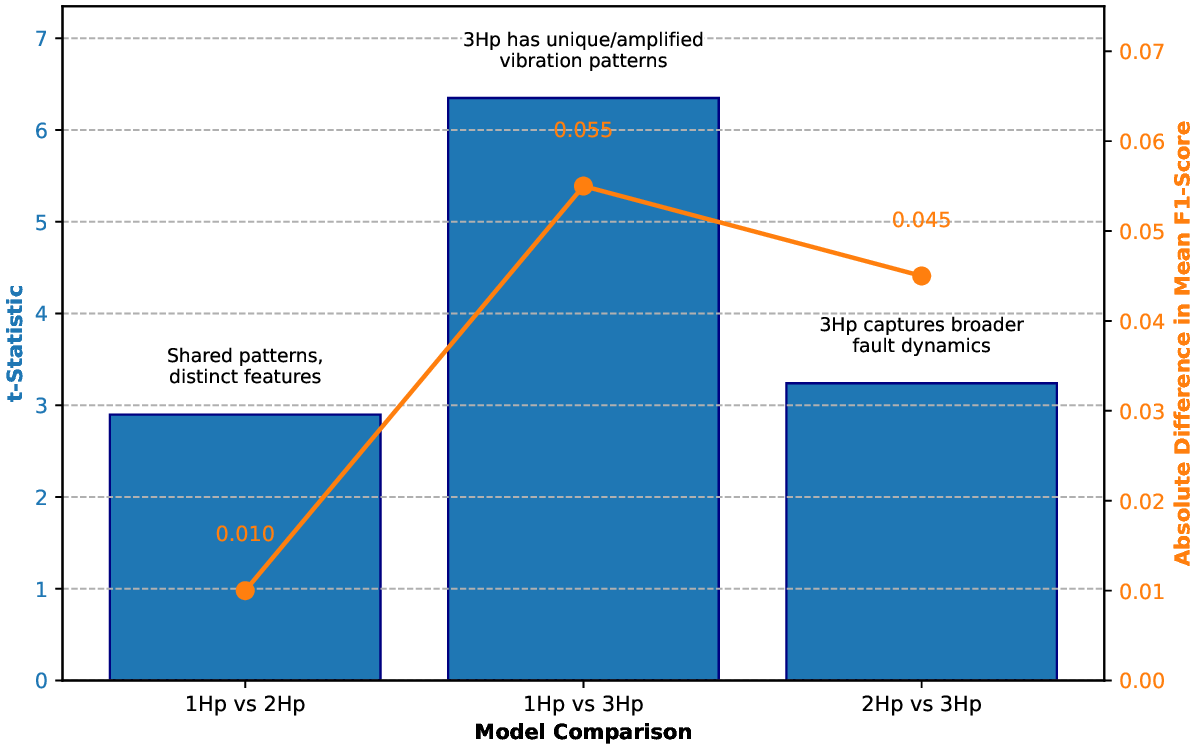}
	\caption{Relationship between dataset characteristics and model performance. Bars represent the t-statistics of dataset comparisons (left axis), while the line shows the absolute difference in mean F1-scores between the corresponding models (right axis). Annotations indicate key dataset differences: 1Hp and 2Hp share patterns but have distinct features (moderate difference), 3Hp has unique/amplified vibration patterns compared to 1Hp (large difference), and 3Hp captures broader fault dynamics than 2Hp (significant difference).}
	\label{fig:dataset_performance_relation}
\end{figure}

In summary, the 3Hp-trained model demonstrated superior classification performance and robustness, with a mean F1-score of 0.995 and minimal variability, making it the most reliable choice for cross-dataset applications. The 2Hp-trained model, with a mean F1-score of 0.950, slightly outperformed the 1Hp-trained model (0.940), though the difference was not statistically significant. As shown in Figure~\ref{fig:dataset_performance_relation}, the 3Hp dataset's richer feature set characterized by amplified vibration patterns and broader fault dynamics enables the 3Hp-trained model to generalize effectively across all test cases. 

\subsection{Comparative Analysis}
\subsubsection{Comparison with State-of-the-Art Models}

Table \ref{table:accuracy_comparison} presents the classification accuracies (\%) of various state-of-the-art models for bearing fault diagnosis across domain adaptation tasks using the 1Hp, 2Hp, and 3Hp datasets. Each column represents a transfer learning scenario where a model is trained on one domain and tested on another?1Hp$\rightarrow$2Hp, 1Hp$\rightarrow$3Hp, 2Hp$\rightarrow$1Hp, 2Hp$\rightarrow$3Hp, 3Hp$\rightarrow$1Hp, and 3Hp$\rightarrow$2Hp?mimicking real-world variations in motor load and operational settings.

We compare traditional models like FFT-SVM \cite{amar2014vibration}, FFT-MLP \cite{saravanan2010incipient}, and FFT-DNN \cite{janssens2016convolutional}, which apply FFT features with classifiers such as SVMs, MLPs, or DNNs. We also evaluate deep learning models, including WDCNN \cite{zhang2017new} and WDCNN (AdaBN) \cite{shenfield2020novel}, which integrate domain adaptation techniques like Adaptive Batch Normalization. Together, we assess the proposed model and the ATBTSGM model \cite{singh2024spatialtemporalbearingfaultdetection}, designed for robust cross-domain performance.

\begin{table}[h!]
	\centering
	\caption{Classification Accuracy (\%) of Different Models Across Dataset Transfer Scenarios}
	\begin{tabular}{|c|c|c|c|c|c|c|}
		\hline
		& \textbf{1Hp$\rightarrow$2Hp} & \textbf{1Hp$\rightarrow$3Hp} & \textbf{2Hp$\rightarrow$1Hp} & \textbf{2Hp$\rightarrow$3Hp} & \textbf{3Hp$\rightarrow$1Hp} & \textbf{3Hp$\rightarrow$2Hp} \\ \hline
		FFT-SVM & 68.6\% & 60.0\% & 73.2\% & 67.6\% & 68.4\% & 62.0\% \\ \hline
		FFT-MLP & 82.1\% & 85.6\% & 71.5\% & 82.4\% & 81.8\% & 79.0\% \\ \hline
		FFT-DNN & 82.2\% & 82.6\% & 72.3\% & 77.0\% & 76.9\% & 77.3\% \\ \hline
		WDCNN & 99.2\% & 91.0\% & 95.1\% & 95.1\% & 78.1\% & 85.1\% \\ \hline
		WDCNN (AdaBN) & 99.4\% & 93.4\% & 97.2\% & 97.2\% & 88.3\% & 99.9\% \\ \hline
		Proposed Model & 96.0\% & 97.0\% & 96.0\% & 95.0\% & 99.9\% & 99.9\% \\ \hline
		ATBTSGM & 99.0\% & 96.0\% & 99.0\% & 96.0\% & 95.0\% & 95.0\% \\ \hline
	\end{tabular}
	\label{table:accuracy_comparison}
\end{table}

\begin{figure}[h!]
	\centering
	\includegraphics[width=0.8\textwidth]{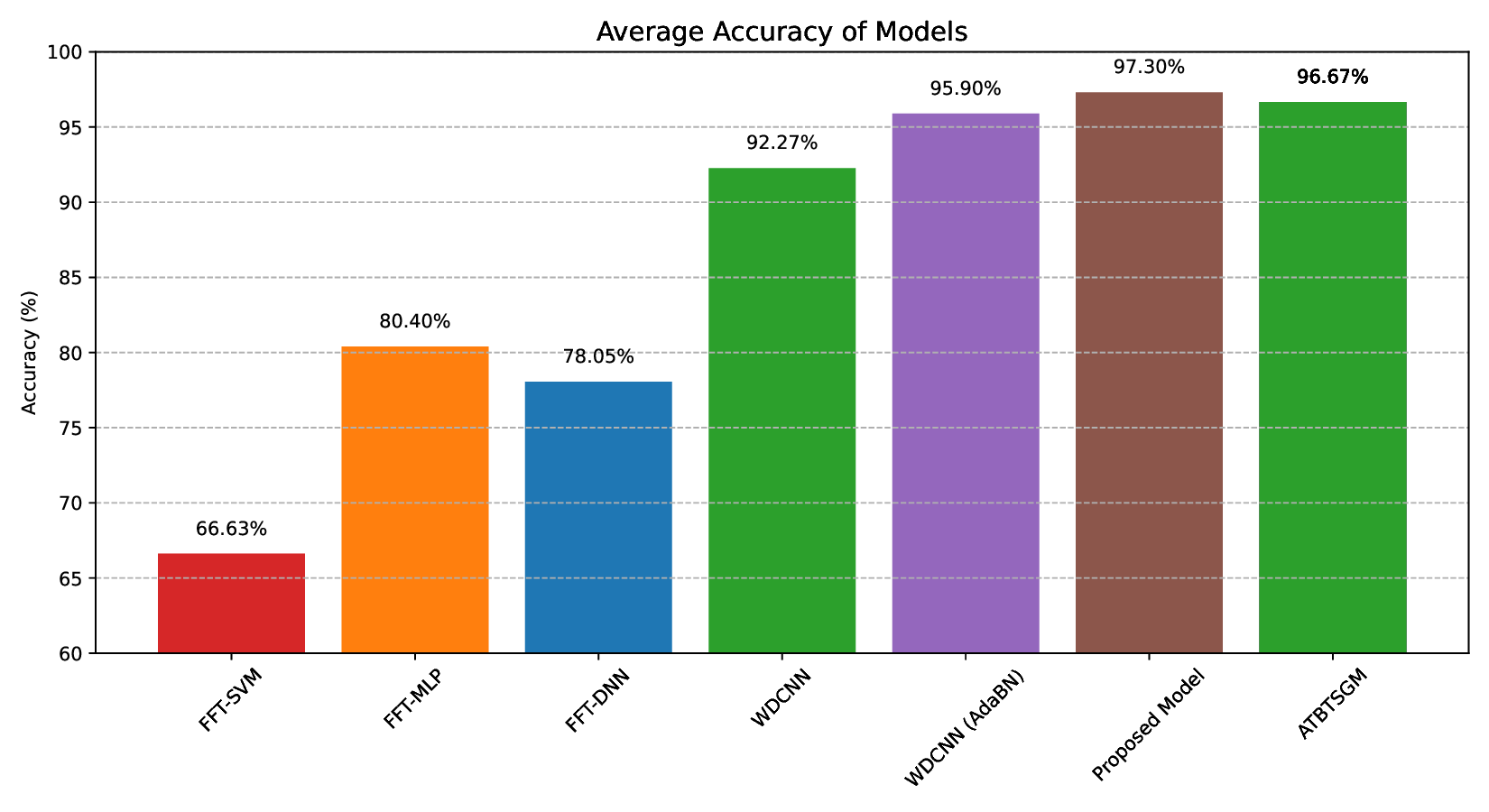}
	\caption{Average accuracy (\%) of different models across all transfer scenarios.}
	\label{fig:avg_accuracy}
\end{figure}

Figure \ref{fig:avg_accuracy} and the per-transfer accuracy values highlight clear performance trends. The proposed model outperforms competing methods in three transfer tasks: it achieves 97.0\% in 1Hp$\rightarrow$3Hp and 99.9\% in both 3Hp$\rightarrow$1Hp and 3Hp$\rightarrow$2Hp. In the 3Hp$\rightarrow$1Hp transfer, the proposed model surpasses ATBTSGM by 4.9\% and WDCNN (AdaBN) by 11.6\%.

In the 1Hp$\rightarrow$3Hp task, the proposed model beats ATBTSGM by 1.0\% and WDCNN (AdaBN) by 3.6\%. It maintains strong performance across all tasks, ranging from 95.0\% to 99.9\%, and achieves the highest overall average accuracy of 97.30\%.

While ATBTSGM performs best in 1Hp$\rightarrow$2Hp and 2Hp$\rightarrow$1Hp (both 99.0\%), it falls to 95.0\% in 3Hp$\rightarrow$1Hp and 3Hp$\rightarrow$2Hp, yielding an average of 96.67\%. WDCNN (AdaBN) excels in 1Hp$\rightarrow$2Hp (99.4\%) and 3Hp$\rightarrow$2Hp (99.9\%) but drops to 88.3\% in 3Hp$\rightarrow$1Hp, resulting in a 95.90\% average. WDCNN's performance also fluctuates significantly, dropping from 99.2\% 

Traditional methods perform the weakest, with FFT-SVM achieving an average accuracy of just 66.63\%, FFT-MLP reaching 80.40\% , and FFT-DNN obtaining 78.05\%, all showing notable performance gaps across tasks. These results underscore the superior generalization capability of the proposed model under domain shifts, making it a practical and reliable solution for real-world fault diagnosis tasks.

\subsubsection{Comparison with Recent Deep Learning Models}

We further evaluate the proposed model against five deep learning baselines - Convolutional Neural Network (CNN), Long Short-Term Memory (LSTM), Recurrent Neural Network (RNN), Gated Recurrent Unit (GRU), and Bidirectional Long Short-Term Memory (Bi-LSTM)
\cite{10504257} on the Case Western Reserve University (CWRU) dataset, which includes 10 bearing condition classes (0-9). We use the F1-score as the primary evaluation metric, as it balances Precision and Recall. GAE results reflect averages over three runs, while baseline scores are taken from the referenced dataset.

\begin{table}[h!]
	\small
	\centering
	\caption{F1-Scores of the proposed model and Deep Learning Models Across 10 Classes on CWRU Dataset}
	\begin{tabular}{|c|c|c|c|c|c|c|}
		\hline
		\textbf{Class} & \textbf{Proposed Model} & \textbf{CNN} & \textbf{LSTM} & \textbf{RNN} & \textbf{GRU} & \textbf{Bi-LSTM} \\ \hline
		0 & 1.00 & 0.84 & 0.95 & 0.93 & 0.91 & 0.95 \\ \hline
		1 & 1.00 & 0.86 & 0.94 & 0.94 & 0.94 & 0.94 \\ \hline
		2 & 0.99 & 0.93 & 0.88 & 0.94 & 0.94 & 0.89 \\ \hline
		3 & 0.99 & 1.00 & 1.00 & 1.00 & 1.00 & 1.00 \\ \hline
		4 & 1.00 & 1.00 & 1.00 & 1.00 & 1.00 & 1.00 \\ \hline
		5 & 1.00 & 0.97 & 0.97 & 0.98 & 0.98 & 0.97 \\ \hline
		6 & 1.00 & 0.98 & 0.97 & 0.97 & 0.97 & 0.97 \\ \hline
		7 & 0.98 & 1.00 & 1.00 & 1.00 & 1.00 & 1.00 \\ \hline
		8 & 0.99 & 0.82 & 0.71 & 0.92 & 0.83 & 0.71 \\ \hline
		9 & 1.00 & 1.00 & 0.98 & 1.00 & 1.00 & 1.00 \\ \hline
		\textbf{Average} & \textbf{0.99} & \textbf{0.94} & \textbf{0.94} & \textbf{0.97} & \textbf{0.96} & \textbf{0.94} \\ \hline
	\end{tabular}
	\label{tab:f1_scores_comparison}
\end{table}

GAE consistently outperforms other models in most classes, achieving perfect F1-scores in Classes 0, 1, 4, 5, 6, and 9. It achieves 0.99 in Classes 2, 3, and 8 and slightly lags in Class 7 (0.98), where all other models score 1.00. Notably, in Class 8?one of the most challenging cases?GAE's F1-score of 0.99 exceeds the best baseline (RNN at 0.92) by 0.07 and the weakest (LSTM and Bi-LSTM at 0.71) by 0.28.

\begin{table}[h!]
	\small
	\centering
	\caption{Statistical Significance Test Results Comparing the proposed model with Baseline Models (Paired t-Test, $\alpha = 0.05$)}
	\begin{tabular}{|p{1.5cm}|c|c|c|p{1.8cm}|}
		\hline
		\textbf{Model} & \textbf{Mean Diff.} & \textbf{t-Statistic} & \textbf{p-Value} & \textbf{Significant?} \\ \hline
		CNN & 0.055 & 2.356 & 0.021 & Yes \\ \hline
		LSTM & 0.055 & 2.006 & 0.038 & Yes \\ \hline
		RNN & 0.027 & 2.547 & 0.016 & Yes \\ \hline
		GRU & 0.038 & 2.201 & 0.028 & Yes \\ \hline
		Bi-LSTM & 0.052 & 1.819 & 0.051 & No \\ \hline
	\end{tabular}
	\label{tab:stat_significance}
\end{table}

Statistical testing using paired t-tests (Table \ref{tab:stat_significance}) confirms that the proposed model significantly outperforms CNN, LSTM, RNN, and GRU at the 0.05 significance level. While the p-value for Bi-LSTM (0.051) slightly exceeds the threshold, it becomes significant under a relaxed level of $\alpha = 0.10$.

We further validate these results using a non-parametric Wilcoxon signed-rank test, which confirms a statistically significant difference between the proposed model and Bi-LSTM ($W = 33, p < 0.05$). Combined with the proposed model's consistent and class-balanced performance, these findings underscore its effectiveness and reliability for real-world bearing fault diagnosis.

\section{Discussion}

The proposed graph-based framework with a Graph Autoencoder (GAE) demonstrates significant advancements in fault classification across datasets with varying operating conditions. By transforming time series vibration data into graph representations, the proposed model captures both local and global temporal patterns-addressing limitations of conventional methods that treat time series as independent sequences. 

Shannon entropy is used to determine the optimal window size for segmentation to ensure informative node representations. This approach results in meaningful segments, particularly in the 3Hp dataset, which exhibits the most diverse fault dynamics. The use of Dynamic Time Warping (DTW) as a similarity metric further strengthens the graph structure by accurately modeling temporal dependencies between segments. These enhancements enable the proposed model to learn robust and discriminative representations for fault classification.

Cross-dataset evaluation reveals performance trends aligned with dataset characteristics. The model trained on the 3Hp dataset achieves the highest mean F1-score of 0.995 with minimal variability (\(\pm 0.005\)), significantly outperforming models trained on 1Hp (0.940, \(\pm 0.042\)) and 2Hp (0.950, \(\pm 0.038\)) datasets. Paired t-tests confirm these differences as statistically significant (\(p = 0.008\) and \(p = 0.005\), respectively). The superior performance is attributed to the 3Hp dataset's amplified vibration signals and broader fault dynamics, offering a richer feature space for training. 

Figure~\ref{fig:dataset_performance_relation} further supports this finding. It shows that larger statistical differences between datasets (e.g., 1Hp vs. 3Hp, \(t = 6.3482\)) correspond to larger F1-score gaps (0.055), highlighting the influence of dataset diversity on generalization. Conversely, the smaller difference between 1Hp and 2Hp (\(t = 2.8976\), F1-score difference = 0.010) suggests shared patterns that support moderate generalization.

The proposed model's encoder architecture, a Graph Attention Network (GAT) and TransformerConv, successfully captures hierarchical relationships within the graph. This design contributes to the consistent cross-dataset performance observed, especially for the 3Hp-trained model. An ensemble classifier is employed for final prediction, further improving robustness by integrating multiple decision boundaries and mitigating overfitting risks.

However, challenges remain in fault classes with subtle patterns, such as class 7, where lower F1-scores are observed (e.g., 0.85?0.89 for the 1Hp-trained model). Enhancing feature extraction for such cases could be a focus for future work, possibly through domain-specific features or advanced attention mechanisms.

One limitation of this study is the relatively small sample size for statistical testing (\(n = 10\) classes), which may affect the power of significance tests. While the Wilcoxon signed-rank test confirms the superiority of the proposed model over Bi-LSTM (\(W = 33, p < 0.05\)), the paired t-test yields a marginal result (\(p = 0.051\)). A larger set of classes or additional fault types could provide stronger statistical validation.

Additionally, the baseline models were evaluated on a single run, whereas the proposed model's F1-scores are averaged over three runs?potentially giving it an advantage in terms of robustness. Future comparisons should include multiple runs for all models to ensure fairness. Moreover, although the CWRU dataset is widely used, it may not reflect the full complexity of real-world industrial environments, where noise, operational variability, and novel fault types are prevalent. Evaluation on more diverse datasets, such as \cite{saeed2025deep}, is recommended to assess real-world applicability.

Finally, the ability of the proposed framework to visualize the relationship between dataset characteristics and model performance offers practical value. For instance, datasets with higher entropy and diverse fault modes, like 3Hp, are more suitable for training models intended for cross-domain applications. This insight is critical in industrial settings where selecting the most representative training data directly impacts the reliability of fault diagnosis systems under varying conditions. Nonetheless, the computational complexity of the model and the sensitivity to hyperparameters?such as the DTW threshold \( \theta \) and regularization parameter \( \lambda \)-highlight important directions for future optimization.

\section{Conclusion}

This study introduced a graph-based framework with a Graph Auto Encoder (GAE) for fault classification across vibration datasets from machinery operating at different horsepower levels. By employing entropy-based segmentation and DTW-based edge construction, we transformed time series data into a graph representation that captures both local and global temporal patterns. The GAE, with its deep graph transformer encoder, decoder, and ensemble classifier, effectively learned a latent representation of the graph, achieving robust cross-dataset generalization. Our evaluation on three datasets demonstrated that the model trained on the most diverse dataset (3Hp) achieved a mean F1-score of 0.995, significantly outperforming models trained on less diverse datasets (1Hp: 0.940, 2Hp: 0.950). The visualization of model performance relative to dataset characteristics provided insights into the impact of vibration pattern complexity and fault dynamics on generalization, highlighting the importance of dataset diversity in training robust models.

The proposed framework addresses the critical challenge of fault classification under varying operating conditions, offering a scalable solution for industrial applications. Future work could explore the integration of domain-specific features to improve performance on subtle fault classes, optimize the computational efficiency of the GAE, and extend the framework to other types of time series data, such as acoustic or thermal signals. Ultimately, this study contributes to the advancement of fault diagnosis by demonstrating the potential of graph-based methods to achieve accurate and generalizable fault classification in complex industrial environments.

\subsection*{Conflict of interest/Competing interests:} The authors declare no conflict of interest or competing interests.
	
\subsection*{Ethics approval and consent to participate:} Not applicable.
	
\subsection*{Consent for publication:} All authors have given their consent for publication.
	
\subsection*{Data availability:} The datasets used and/or analyzed during the current study are available from the corresponding author upon reasonable request.
	
\subsection*{Materials availability:} Not applicable.
	
	\subsection*{Code availability:} The code used in this study is available from the corresponding author upon reasonable request.

%
%
%
%

\bibliographystyle{sn-mathphys-num}
\bibliography{sn-bibliography}



\end{document}